\documentclass{new_tlp}

\usepackage{soul}
\usepackage{amsmath}
\usepackage{graphicx}
\usepackage{multirow}
\usepackage[english]{babel}

\let\proof\relax

\usepackage{amssymb}
\usepackage{amsthm}
\newtheorem{theorem}{Theorem}
\newtheorem{definition}{Definition}
\newtheorem{example}{Example}

\usepackage{alltt}
\usepackage{listings}
\usepackage{cleveref}
\usepackage{multicol}
\usepackage{url}
\usepackage{xstring}
\usepackage{xspace}

\newcommand{\fod}{FO(\(\cdot\))}
\newcommand{\fodot}{\fod{}}

\newcommand{\asp}[1]{\dot{#1}}
\newcommand{\aspcore}{ASP-Core-2}

\DeclareMathOperator{\lif}{{:}{-}{~}}
\newcommand{\aspnot}{\mathit{not}~}
\newcommand{\aspcard}{\#\mathit{count}}

\usepackage[]{todonotes}

\newcommand{\theo}{\mathit{Th}}
\newcommand{\defarrow}{\gets}

\newcommand{\country}{\mathit{country}}
\newcommand{\Country}{\mathit{Country}}
\newcommand{\ccolor}{\mathit{color}}
\newcommand{\Color}{\mathit{Color}}
\newcommand{\border}{\mathit{border}}
\newcommand{\Border}{\mathit{Border}}
\newcommand{\symBorder}{\mathit{symBorder}}
\newcommand{\SymBorder}{\mathit{SymBorder}}
\newcommand{\colorOf}{\mathit{colorOf}}
\newcommand{\ColorOf}{\mathit{ColorOf}}
\renewcommand{\phi}{\varphi}
\newcommand{\folasp}{\textsc{FOLASP}}
\newcommand{\idpthree}{\textsc{IDP}}
\newcommand{\idp}{\idpthree{}}
\newcommand{\clingo}{\textsc{Clingo}}
\newcommand{\clasp}{\clingo{}}
\newcommand{\gitcommit}[1]{\StrLeft{#1}{8}\xspace}
\newcommand{\instance}[1]{\textit{#1}\xspace}
\newcommand{\typeguard}[1]{\asp{T}(\asp{\vec{#1}})}

\usepackage{babel}
\begin{document}

\title[FOLASP]{\folasp{}: \fod\ as Input Language for Answer Set Solvers}
			

\author[K. Van Dessel, J. Devriendt and J. Vennekens]
         {KYLIAN VAN DESSEL, JO DEVRIENDT \and JOOST VENNEKENS\\
         KU Leuven, Dept. of Computer Science, De Nayer Campus, Sint-Katelijne-Waver, Belgium\\
         Leuven.AI -- KU Leuven Institute for AI, Leuven, Belgium\\
         \email{firstname.lastname@kuleuven.be}}

\maketitle


\abstract

Over the past decades, Answer Set Programming (ASP) has emerged as an important paradigm for declarative problem solving. Technological progress in this area has been stimulated by the use of common standards, such as the \aspcore{} language. While ASP has its roots in non-monotonic reasoning, efforts have also been made to reconcile ASP with classical first-order logic (FO). This has resulted in the development of \fod{}, an expressive extension of FO, which allows ASP-like problem solving in a purely classical setting. This language may be more accessible to domain experts already familiar with FO, and may be easier to combine with other formalisms that are based on classical logic. It is supported by the IDP inference system, which has successfully competed in a number of ASP competitions. Here, however, technological progress has been hampered by the limited number of systems that are available for \fod{}. In this paper, we aim to address this gap by means of a translation tool that transforms an \fod{} specification into \aspcore{}, thereby allowing \aspcore{} solvers to be used as solvers for \fod{} as well. We present experimental results to show that the resulting combination of our translation with an off-the-shelf ASP solver is competitive with the IDP system as a way of solving problems formulated in \fod{}. 

Under consideration for acceptance in TPLP.


\section{Introduction}
\label{sec:intro}

Answer set programming (ASP) is a knowledge representation (KR) paradigm in which a declarative language is used to model and solve combinatorial (optimization) problems~\cite{marek99stable}. It is supported by performant solvers~\cite{tplp/GebserMR20}, such as Clingo~\cite{iclp/GebserKKOSW16} and the DLV system~\cite{tocl/LeonePFEGPS06}. Development and use of these solvers has been simplified and encouraged by the emergence of the unified \aspcore{} standard~\cite{tplp/CalimeriFGIKKLM20}.

The roots of ASP lie in the area of non-monotonic reasoning and its semantics is defined in a non-classical way. Work by Denecker et al.~\cite{tocl/DeneckerT08} has attempted to integrate key ideas from ASP with classical first-order logic (FO), in an effort to clarify ASP's contributions from a knowledge representation perspective. This has resulted in the development of the language \fodot{}, pronounced ``ef-oh-dot'', which is a conservative extension of FO. This language may be easier to use for domain experts who are already familiar with FO than ASP, and can seamlessly be combined with other monotonic logics. A number of systems, such as the IDP system~\cite{WarrenBook/DeCatBBD14} and Enfragmo~\cite{phd/Aavani14}, already support the \fodot{} language. However, when compared to the variety of solvers for \aspcore, the support for \fodot{} is still rather limited. This hinders technological progress, both in terms of solver development and the development of applications.

In this paper we present \emph{\folasp{}}, a tool that translates \fodot{} to \aspcore{}, thereby allowing each solver that supports \aspcore{} to handle \fodot{} as well. In this way, we significantly extend the range of solvers that is available for \fodot{}.

We believe that this tool will make the \fod{} language more accessible and useful for practical applications, while also helping to drive technological progress. To develop \folasp{}, we build on fundamental results about the relation between ASP and \fod{}~\cite{jelia/MarienGD04,DLTV12Tarskian}, which we have for the first time combined into a working tool.


\section{Preliminaries}

\subsection{\fod{}}
\label{sec:preidp}

\fod{} is an extension of classical typed first-order logic with aggregates, arithmetic, and (inductive) definitions.
To maximize clarity, we will consider only a core subset of \fodot{}: typed FO extended with definitions, cardinality aggregates, and comparison operators.


A vocabulary \(V\) consists of a set of types $T$, predicates $P$ and function symbols $F$. 
Each predicate \(P/n\) with arity $n$ has an associated typing \(\tau(P)=(T_1,\ldots,T_n)\), as has each function symbol $F/n$: \(\tau(F) = (T_1,\ldots,T_{n+1})\).



A structure \(S\) for a vocabulary \(V\) (also known as a $V$-structure) consists of a \emph{domain} \(D\) and an appropriate \emph{interpretation} \(\sigma^S\) for each symbol \(\sigma \in V\).
The interpretation \(T^S\) of a type \(T\) is a subset of $D$, the interpretation \(P^S\) of a predicate $P$ with \(\tau(P) = (T_1,\ldots,T_n)\) is a relation $P^S \in T_1^S \times \cdots \times T_n^S$, and the interpretation \(F^S\) of a function symbol $F$ with \(\tau(F) = (T_1,\ldots,T_n, T_{n+1})\) is a function from \(T_1^S \times \cdots \times T_n^S\) to \(T_{n+1}^S\). 
The interpretations \(T_i^S\) of the types \(T_i \in V\) partition the domain \(D\), i.e., \(\bigcup_i T_i^S = D\) and \(T_i^S \cap T_j^S = \emptyset \) for \(i \neq j\).


%
%
A \emph{term} is either a variable, an integer, a function $f(\vec{t})$ applied to a tuple of terms $\vec{t}$, or a cardinality expression of the form $\# \{ \vec{x} \colon \varphi(\vec{x}) \}$, which intuitively represents the number of $\vec{x}$'s for which $\phi(\vec{x})$ holds. 
Note that cardinality expressions are a special case of a aggregate expressions, which sometimes are introduced as generalized quantifiers working as atoms.
Here, we use the FO terminology, where a cardinality expression is a term.

We use the notion of a \emph{simple term} to refer to a variable or an integer.
An \emph{atom} is either a predicate $P(\vec{t})$ applied to a tuple of terms or a comparison $t_1 \bowtie t_2$ between two terms, with ${\bowtie}\in\{=,\neq,\leq,\geq,<,>\}$.
As usual, \emph{formulas} are constructed by means of the standard FO connectives $\lnot, \lor, \land, \Rightarrow, \Leftrightarrow, \exists, \forall$. Only well-typed formulas are allowed. A \emph{sentence} is a formula without free variables.
A \emph{positive literal} is an atom, a \emph{negative literal} a negated atom.

As in FO, an \fod{} theory can be a set of sentences. However, in addition to sentences, \fod{} also allows \emph{definitions}. Such a definition is a set of rules of the form:
\[
  \forall x_1, \ldots, x_n\colon P(x_1, \ldots, x_n) \defarrow \varphi(x_1, \ldots, x_n).
\]
where \(P/n\) is a predicate symbol, \(x_1,\ldots,x_n\) variables, and \(\varphi\) a formula.
The atom \(P(x_1,\ldots,x_n)\) is the \emph{head} of the rule, while \(\varphi\) is the \emph{body}.  The purpose of such a definition is to define the predicates that appear in the heads of the rules in terms of the predicates that appear only in the body. The first kind of predicates are called the \emph{defined predicates} $Def(\Delta)$ of the definition $\Delta$, while the second are called its \emph{open predicates} $Open(\Delta)$.

The formal semantics of these definitions is given by a parametrized variant of the well-founded semantics~\cite{pods/GelderRS88}. In order for a definition to be valid in \fod{}, it must be such that it uniquely determines a single interpretation for the defined predicates, given any interpretation for the open predicates. Formally, the condition is imposed on definitions that their well-founded model must exist and always be two-valued, no matter what the interpretation for their open predicates might be.




Different logical inference tasks can be considered for \fod{}.
In this paper, we focus on the most common task, namely that of \emph{model expansion}.

\begin{definition}
\label{def:mx}
Let \(\theo\) be a theory over vocabulary \(V\) and \(S\) a structure for a subvocabulary \(Voc(S) \subseteq V\). The \emph{model expansion} problem \(MX(V,S,\theo)\) is the problem of computing a \(V\)-structure \(S' \supseteq S\) such that \(S' \models \theo\).
\end{definition}

\begin{example}
\label{ex:vocabulary}
The following example models the well-known graph coloring problem as a model expansion problem $MX(V,S,\theo)$, with an illustration of a definition for the symmetric closure of the border relation.
\begin{align*}
V \colon & \mathit{type~Country},~ \mathit{type~Color},\\ 
& \mathit{predicate~Border~with~} \tau(\Border)=(\Country,\Country), \\
& \mathit{predicate~SymBorder~with~} \tau(\SymBorder)=(\Country,\Country), \\
& \mathit{function~symbol~ColorOf~with~typing~} \tau(\ColorOf)=(\Country,\Color) \\
S \colon
& \Country^S =\{ be, nl, lux \} \\
& \Color^S = \{ red, blue \} \\
& \Border^S = \{ (nl,be), (be,lux) \}  \\
\theo \colon
    & \forall c_1,c_2\colon \Border(c_1,c_2) \Rightarrow \ColorOf(c_1) \neq \ColorOf(c_2) \\
  & \left\{ 
    \begin{array}{l}
     \forall c_1,c_2\colon \SymBorder(c_1,c_2) \defarrow \Border(c_1,c_2). \\
     \forall c_1,c_2\colon \SymBorder(c_1,c_2) \defarrow \SymBorder(c_2,c_1). 
     \end{array}
     \right\}
\end{align*}

One solution to $MX(V,S,\theo)$ is 
\begin{align*}
S' \colon
& \Country^{S'} = \Country^{S}, \Color^{S'} = \Color^{S}, \Border^{S'} = \Border^{S} \\
& \SymBorder^{S'} = \{ (nl,be), (be,nl), (be,lux), (lux,be) \} \\
& \ColorOf^{S'} = \{ be \mapsto red, nl \mapsto blue, lux \mapsto blue \}
\end{align*}
\end{example}


\subsection{ASP}
\label{sec:preasp}

A normal logic program is a set of rules of the form:
\begin{equation}
\label{eq:hb}
  H \lif B_1, \ldots, B_m, not~ B_{m+1}, \ldots, not~ B_n.
\end{equation}
Here, \(H\) and all \(B_i\) are atoms.
Corresponding with definitions, \(H\) is called the \emph{head} of the rule, while the conjunction \(B_1, \ldots, not~ B_n\) is called the \emph{body}.
Both the rule head and body can be empty.
A rule with empty head (= false) is called a \emph{constraint}, a rule with empty body (= true) a \emph{fact}.

The semantics of a program is defined in terms of its \emph{grounding} which is an equivalent program without any variables, so all atoms are \emph{ground atoms}. 
An interpretation $I$ is a set of ground atoms.
A rule of form (\ref{eq:hb}) is satisfied in $I$ if $H \in I$ whenever $B_1, \ldots, B_m \in I$ and $B_{m+1}, \ldots, B_n \not \in I$.
An interpretation is a model of a program if it is a model of each rule.
The reduct of a program \(P\) w.r.t. interpretation \(I\), denoted \(P^I\), contains the ground rule \(H \lif B_1, \ldots, B_m\) for each rule of form (\ref{eq:hb}) for which none of the atoms \(B_{m+1}, \ldots, B_n\) belong to \(I\). An interpretation \(I\) is a \emph{stable model} or \emph{answer set} of program \(P\) if it is a minimal model of \(P^I\).

The \aspcore{} language extends this basic formalism in a number of ways.
For instance, it includes \emph{choice rules} that can be used to express that a certain atom $H$ \emph{may} be true:
\begin{equation*}
  \{H\} \lif B_1, \ldots, B_m, not~ B_{m+1}, \ldots, not~ B_n.
\end{equation*}
Choice rules allow to generate a search space of candidate answer sets, from which the desired solutions can be filtered out with constraints.

A second extension is the cardinality aggregate, which, as in \fod{}, counts the size of the set of free variable instantiations for which a conjunction of atoms holds.
Throughout this paper, we make use of cardinality aggregates in body atoms, which have the form:
\[
\aspcard\{\vec{X} \colon B_1, \ldots, B_m, not~ B_{m+1}, \ldots, not~ B_n\} \bowtie t
\]
with ${\bowtie}~\in \{=,\neq,\leq,\geq,<,>\}$ and $t$ a simple term.

\begin{example}
\label{ex:asp}
The following ASP program is the counterpart of the \fod{} graph coloring model expansion problem from the previous section:
\begin{align}
& \country(be).~\country(nl).~\country(lux). \nonumber{}\\ 
& \border(nl,be).~\border(be,lux).
~\ccolor(\mathit{red}).~\ccolor(\mathit{blue}). \label{ex:asp:facts}\\
\{\colorOf(C,X)\} \lif & \country(C), \ccolor(X). \label{ex:asp:generate} \\
\lif & \aspcard \{C,X\colon \colorOf(C,X),\ccolor(X)\}\neq 1, \country(C). \label{ex:asp:card}\\
\lif & \border(C_1,C_2), \colorOf(C_1,X), \colorOf(C_2,X). \label{ex:asp:constraint}\\
\symBorder(C_1,C_2) \lif & \border(C_1,C_2). \nonumber{}\\
\symBorder(C_1,C_2) \lif & \symBorder(C_2,C_1). \label{ex:asp:symborder}
\end{align}
\end{example}


\section{Translation of \fod\ to ASP}
\label{sec:transl}


In this section, we define a translation $\alpha$ from an \fod{} model expansion problem $M = MX(V,S,\theo)$ to an ASP program $\alpha(M)$.
The translation $\alpha$ consists of four components $\alpha_1, \alpha_2, \alpha_3$ and $\alpha_4$. $\alpha_1$ and $\alpha_2$ translate $V$ and $S$, respectively (discussed in Section~\ref{ssec:search_space}). $\alpha_3$ and $\alpha_4$ translate the \fod{} sentences and definitions that belong to $\theo$, respectively (Section~\ref{ssec:form} and \ref{ssec:def}). Prior to translation we normalize the specification in order to make it compatible for translation to ASP (Section~\ref{ssec:norm}).


\subsection{Normalization of an \fod{} specification}
\label{ssec:norm}
As a first step, we normalize the specification of the model expansion problem $MX(V,S,\theo)$.
%
Firstly, we convert all formulas in \(\theo\) to negation normal form (NNF)~\cite{ENDERTON02}, i.e., the Boolean operators are restricted to negation (\(\neg\)), conjunction (\(\land\)) and disjunction (\(\lor\)), 
and the negation operator is only applied directly to atoms.
We also assume that the type $T$ of any variable $x$ is known, either by automated type derivation or by an explicit annotation, e.g., $\forall x [T] \colon \varphi$.




In \fod{}, a function symbol can be interpreted by any function of the appropriate arity and type.
By contrast, in ASP, each function symbol $F$ is interpreted by the Herbrand function that maps each tuple of arguments $\vec{t}$ to the syntactic term $F(\vec{t})$.
Hence, we eliminate function symbols from the \fod{} specification.
For this, we first rewrite the theory such that function symbols only appear in atoms of the form $F(\vec{x})=y$ with $\vec{x}$ and $y$ simple terms.
This is done by recursively replacing a (negated) atom $(\lnot) A$ with subterm $F(\vec{t})$ by the NNF equivalent of
\begin{equation*}
    \forall x\colon F(\vec{t})=x \Rightarrow (\lnot) A[F(\vec{t})/x]
\end{equation*}


In a similar way, we \emph{unnest} cardinality terms, such that these only occur in a comparison atom with simple terms, i.e., $\#\{\vec{x} \colon \varphi(\vec{x})\} \bowtie y$ with $y$ a simple term.
After this, each term $t$ that is not simple occurs only in equality atoms $t=x$ (if $t$ is a function symbol application) or comparison atoms $t\bowtie x$ (if $t$ is a cardinality expression), with $x$ a simple term.

%
Each function symbol \(F/n\) is then transformed into a predicate \(P_F/{n{+}1}\) with the same typing, i.e., \(\tau(F) = (T_1,\ldots,T_n,T_{n+1}) = \tau(P_F)\). 
We replace the atoms $F(\vec{x})=y$ in the theory with $P_F(\vec{x}, y)$.
We also add the constraint implied by using a function symbol, i.e., that each tuple of arguments has exactly one image, to the theory:
  \[
    \forall x_1 \ldots\ x_n\colon \#\{x_{n+1} \colon P_F(x_1, \ldots, x_n, x_{n+1})\}=1.
  \]
If the structure $S$ has an interpretation $F^S$, then we replace it by
\begin{equation*}
P_F^S=\{(a_1,\ldots,a_{n+1}) ~|~ (a_1,\ldots,a_{n})\mapsto a_{n+1} \in F^S\} \text{.}
\end{equation*}

Finally, after eliminating function symbols, we push negations through atoms of the form $t \bowtie x$ by adjusting the $\bowtie$ operator.
After this, the only negated atoms remaining have the form $\neg P(\vec{x})$ for some predicate $P$ and simple terms $\vec{x}$.

\begin{example}
\label{ex:normal}

The normalized \fod{} graph coloring theory from Section~\ref{sec:preidp} is:
\begin{align*}
\theo'\colon
& \forall x[\Country]\colon \#\{y[\Color] \colon \ColorOf(x,y)\}=1 \\
& \forall c_1[\Country]\colon \forall c_2[\Country]\colon \lnot \Border(c_1,c_2) \lor \\
& ~~~~ \forall x[\Color]\colon \lnot\ColorOf(c_1,x) \lor \lnot\ColorOf(c_2,x)\\
& \left\{ 
    \begin{array}{l}
    \forall c_1,c_2\colon \SymBorder(c_1,c_2) \defarrow \Border(c_1,c_2). \\
    \forall c_1,c_2\colon \SymBorder(c_1,c_2) \defarrow \SymBorder(c_2,c_1). \\
    \end{array}
  \right\}
\end{align*}
\end{example}


\subsection{Generating the search space}
\label{ssec:search_space}

The solutions to a model expansion problem $MX(V,S,\theo)$ are to be found among the set of all structures for $V$ that expand $S$ -- the \emph{search space}. In this section, we translate vocabulary $V$ and structure $S$ to generate precisely this search space.

Unlike \fod{}, \aspcore{} imposes strict naming conventions: variable names must start with a capital, while the names of all other kinds of symbols must start with a lower case letter. For an \fodot{} type, predicate, variable or domain element \(\sigma\), we denote by \(\asp{\sigma}\) a corresponding ASP symbol of the right kind.
Naturally, we enforce that \(\asp{\sigma} \neq \asp{\sigma}'\) whenever \(\sigma \neq \sigma'\).



In a model expansion problem $M = MX(V,S,\theo)$, some of the symbols in $V$ are interpreted by $S$ (we denote these by $Voc(S)$), while others (i.e., $V\setminus Voc(S)$) are not. Because types must always be interpreted by $S$ and our normalization step transforms all function symbols into predicates, the uninterpreted vocabulary $V\setminus Voc(S)$ consists entirely of predicates.  

We translate each \(P \in V\setminus Voc(S)\) with associated typing \(\tau(P) = (T_1,\ldots,T_n)\)  into the following ASP choice rule \(\alpha_1(P)\):
\[
\{\asp{P}(X_1,\ldots, X_n)\} \lif \asp{T_1}(X_1),\ldots,\asp{T_n}(X_n).
\]
This provides a first component $\alpha_1(M)$ of our translation of the model expansion problem. The second component $\alpha_2(M)$ translates the structure $S$. 



%
We translate an interpretation 
$P^S = \{ (a_1^1,\ldots, a_n^1), \ldots,   (a_1^m,\ldots, a_n^m)\}$ of an $n$-ary predicate $P$ 
into the following set \(\alpha_2(P^S)\) of \(m\) ASP facts:
\[
\asp{P}(\asp{a_1^1},\ldots, \asp{a_n^1}).
~~~ \ldots ~~~
\asp{P}(\asp{a_1^m},\ldots,\asp{a_n^m}).
\]
The interpretation of a type is translated as though it were a unary predicate. The second component $\alpha_2(MX(V,S,\theo))$ of our translation now consists of all  $\alpha_2(P^S)$ for which $P \in Voc(S)$. Together, $\alpha_1$ and $\alpha_2$ allow us to generate the correct search space in ASP, as the following theorem shows.





\begin{theorem}\label{th:space}
  For each structure $S$ for a subvocabulary \(Voc(S) \subseteq V\), if $A$ is the ASP program $\alpha_1(V\setminus Voc(S)) \cup \alpha_2(S))$, then $MX(V,S,\{\}) = AnswerSets(A)$. 
\end{theorem}

Here, the equality between structures and answer sets is of course modulo a straightforward ``syntactic'' transformation: we can transform each structure $S$ to the answer set $f(S)$ that consists of all atoms $P(\vec{d})$ for which $\vec{d} \in P^S$. Because we consider a typed logic, in which each element of the domain of $S$ must belong to the interpretation $T^S$ of at least one type $T$, this transformation is an isomorphism, which we omit from our notation for simplicity.

\begin{proof}
  The set $MX(V,S,\{\})$  consists of all $V$-structures $S'$ that can be constructed by starting from the structure $S$ and then adding, for each predicate $P \in V\setminus Voc(S)$ with type $\tau(P) =  (T_1, \ldots, T_n)$, any set of tuples $\subseteq T_1^S\times \cdots\times T_n^S$ as interpretation $P^{S'}$ of $P$ in $S'$. For each predicate $P$ interpreted by $S$, $\alpha_2(S)$ contains precisely all facts $P(\vec{d})$ for which $\vec{d} \in P^S$. Moreover, $P$ does not appear in $\alpha_1(V \setminus Voc(S))$. This ensures that each answer set in $AnswerSets(A)$ contains precisely all atoms $P(\vec{d})$ for which $\vec{d} \in P^S$. For each predicate $P \in V\setminus Voc(S)$, $S'$ may have any set of tuples $\subseteq T_1^S\times \cdots\times T_n^S$ in its interpretation. The choice rules in $\alpha_2(S)$ ensure that precisely these tuples also make up the possible interpretations for $P$ in  $AnswerSets(A)$.
\end{proof}


\subsection{Translating formulas}
\label{ssec:form}

We now define a third component $\alpha_3(MX(V,S,\theo))$ of our translation to transform the formulas $\phi \in \theo$ to ASP. We start by the base case: the translation $\alpha_3(A)$ for an atom $A$. Due to our normalization step, the only atoms $A$ that appear in $\theo$ are of the form $P(\vec{x})$ (with $P/n$ a predicate and $\vec{x}$ a tuple of simple terms) or $t \bowtie x$ (with $t$ a simple term or a cardinality expression, $\bowtie$ a comparison operator, and $x$ a simple term).
%
%
%
%
%
The translation $\alpha_3(P(x_1,\ldots,x_n))$ of a predicate atom is the ASP conjunction
\begin{equation}
\label{eq:atom}
\asp{P}(\asp{x_1},\ldots, \asp{x_n}), \asp{T_1}(\asp{x_1}), \ldots, \asp{T_n}(\asp{x_n})
\end{equation}
with typing $\tau(P) = (T_1,\ldots,T_n).$
Hence, the type information implicit in the typing of a predicate is added explicitly by means of the additional conjuncts $\asp{T_i}(X_i)$.
With slight abuse of notation, we shorten such a conjunction of type atoms to $\typeguard{x}$.

With $x$ and $y$ simple terms, the translation $\alpha_3(x \bowtie y)$ is the ASP conjunction
\begin{equation}
\label{eq:bowt}
\asp{x}\bowtie\asp{y},\asp{T}(\asp{x}),\asp{T}(\asp{y})\text{.}
\end{equation}

The translation $\alpha_3(\#\{\vec{x}[\vec{T}] \colon \varphi(\vec{x},\vec{y})\} \bowtie z)$ of a normalized cardinality atom is
\begin{equation}
\label{aspcard}
\aspcard \{ \asp{\vec{x}} \colon \delta(\asp{\vec{x}},\asp{\vec{y}}), \typeguard{x} \} \bowtie \asp{z}, \typeguard{y}, \asp{T}(\asp{z})
\end{equation}
with $\delta$ a fresh auxiliary predicate representing the subformula $\varphi$. 
Hence, we also add the rule:
\begin{equation}
\label{deltaagg}
\delta(\asp{\vec{x}},\asp{\vec{y}}) \lif \alpha_3(\varphi(\vec{x},\vec{y})), \typeguard{x}, \typeguard{y}.
\end{equation}
with a recursive application of $\alpha_3$.

After normalization, a negation occurs only in literals of the form $\neg P(\vec{x})$, whose translation $\alpha_3(\neg P(\vec{x}))$ simply is
\begin{equation}
\label{aspneg}
\aspnot \asp{P}(\asp{\vec{x}}), \asp{T}(\asp{\vec{x}})\text{.}
\end{equation}
Note that the $\aspnot$ is only added to the first atom and not to the type atoms.

Having defined how each (negated) atom $(\lnot) A$ is translated into a corresponding ASP expression $\alpha_3((\lnot) A)$, we now inductively define how more complex formulas are translated.

The translation $\alpha_3(\varphi \wedge \psi)$ of a conjunction is the ASP conjunction $\alpha_3(\varphi),\alpha_3(\psi)$.

The translation $\alpha_3(\varphi(\vec{x}) \vee \psi(\vec{y}))$ of a disjunction is the ASP atom $\delta(\asp{\vec{x}},\asp{\vec{y}})$, with $\delta$ a fresh auxiliary predicate.
Additionally, for each such auxiliary predicate, we add the following ASP rules:
\begin{align}\label{deltadisj}
\begin{gathered}[b]
  \delta(\asp{\vec{x}},\asp{\vec{y}}) \lif \alpha_3(\varphi(\vec{x})), \typeguard{x},\typeguard{y}.\\
  \delta(\asp{\vec{x}},\asp{\vec{y}}) \lif \alpha_3(\psi(\vec{y})), \typeguard{x},\typeguard{y}.
\end{gathered}
\end{align}
to ensure that $\delta$ indeed corresponds to the disjunction of $\varphi$ and $\psi$.

Since a variable that appears in the body of an ASP rule but not in its head is implicitly existentially quantified, the translation $\alpha_3(\exists x[T]\colon \varphi(x))$ of an existential quantification is the ASP conjunction $\alpha_3(\varphi(x)), \asp{T}(\asp{x})$.

The translation $\alpha_3(\forall x[T]\colon \varphi(x))$ of a universally quantified formula is the ASP cardinality atom $\alpha_3(\#\{x[T]\colon \varphi(x)\}) = n$, with $n = \lvert T ^S\rvert$ the number of elements in type $T$.
Note that, because of this step, the translation $\alpha_3$ not only depends on the theory $\theo$ of our model expansion problem, but also on the structure $S$. 
A first way to avoid this dependence is to translate $\forall x[T]\colon \varphi(x)$ as $\alpha_3(\#\{x[T]\colon \lnot \varphi(x)\}) = 0$.
However, this would introduce an additional negation, which might lead to the introduction of loops over negation in Section~\ref{ssec:def}. 
A second way introduces an aggregate term representing $n$, e.g., $\alpha_3(\#\{x[T]\colon \varphi(x)\} = \#\{x[T]: true\})$, but we expect this to be less efficient.

For any formula $\phi(\vec{x})$, we can now use the transformation $\alpha_3$ to define a fresh ASP symbol $\delta_\phi(\asp{\vec{x}})$ such that the set of all $\vec{x}$ for which $\phi({x})$ holds in the \fod{} theory $\theo$ coincides with the set of all $\asp{\vec{x}}$ for which $\delta_\phi(\asp{\vec{x}})$ holds in the ASP program. We do this by adding the following \emph{reification rule} $r_{\phi}$:
\[\delta_\phi(\asp{\vec{x}}) \lif \alpha_3(\phi(\vec{x})),\typeguard{x}.\]

\begin{theorem}\label{th:form}
Let $\phi$ be a formula in vocabulary $V$ and let $S$ be a structure for $V$. Consider the ASP program $R_\phi$ that consists of all reification rules $r_\psi$ for which $\psi$ is a subformula of $\phi$, together with all additional rules produced by the translation $\alpha_3$ (see Eq.~\ref{deltaagg} and Eq.~\ref{deltadisj}). Let $R_S = \alpha_2(S)$ be the translation of the structure $S$. Then $R_S \cup R_\phi$ has a unique answer set $\mathbb{A}$ and for each subformula $\psi$ the set of all $\vec{d}$ for which $S \models \psi(\vec{d})$ is equal to the set of all $\vec{d}$ for which $\delta_\psi(\vec{d}) \in \mathbb{A}$.
\end{theorem}
\begin{proof}
  $R_S$ is a set of facts over $V$. $R_\phi$ is a strictly stratified set of rules with non-empty heads. Therefore, it is clear that the answer set of $R_S \cup R_\phi$ is indeed unique. We now prove the theorem by induction over the subformula order. The base cases are atoms as translated in Eq.~\ref{eq:atom} and Eq.~\ref{eq:bowt}. Here, it is obvious from the translation that the correspondence holds. For an aggregate (Eq.~\ref{deltaagg}), we can apply the induction hypothesis to obtain a correspondence between the tuples for which $\varphi$ in the original aggregate holds and the tuples for which the fresh predicate $\delta$ in its translation holds; from this, the result follows. Similarly, the case for disjunction follows from applying the induction hypothesis to the fresh predicates in Eq.~\ref{deltadisj}. The cases for negation, conjunction and existential quantification are trivial. The case for universal quantification follows immediately from the correctness of the translation of aggregates.
\end{proof}

Once we have the reification rules $r_\phi$ as defined above, we can eliminate answer sets in which the formula $\phi$ is not satisfied by adding a constraint $\lif \aspnot \delta_\phi$. Denoting such a constraint by $C_\phi$, the third component of our translation -- the translation of sentences -- now is \[ \alpha_3(MX(V,S,\theo)) = \{r_\phi\mid \phi \in \theo\} \cup R \cup \{C_\phi \mid \phi \in \theo\},\] where the rules $r_\phi$ and constraints $C_\phi$ are as above and $R$ are all of the additional rules (see Eq.~\ref{deltaagg} and Eq.~\ref{deltadisj}) generated by producing the $r_\phi$.

\begin{theorem}
  Let $M$ be a model expansion problem $MX(V,S,\theo)$ in which $\theo$ is a set of FO sentences. The solutions to $M$ coincide with the answer sets of $\alpha_1(M) \cup \alpha_2(M) \cup \alpha_3(M)$.
\end{theorem}
\begin{proof}
 By induction on the size of $\theo$. The base case in which $\lvert \theo \rvert = 0$ and therefore $\theo = \{\}$ is covered by Theorem \ref{th:space}. Once the induction hypothesis gives us the correspondence between a theory $\theo$ of size $n-1$ and an ASP program $A_{n-1}$, we can add an additional formula $\phi_n$ and prove the correspondence between $\theo \cup \{\phi_n\}$ and $A_n = A_{n-1} \cup \{r_{\phi_n}, C_{\phi_n}\} \cup R$, with $R$ the additional rules for producing $\alpha_3(\phi_n)$. The atoms in the head of the new rules $\{r_{\phi_n}, C_{\phi_n}\} \cup R$ are all fresh atoms that do not appear in $A_{n-1}$. Therefore there can be no interference  between the new rules and the old ones, and the result follows from Theorem \ref{th:form}.
\end{proof}

We now have a translation for theories that consists entirely of FO sentences. The next section examines how we can extend this to \fod{} theories that contain also definitions.

\subsection{Translating definitions}
\label{ssec:def}


In general, a theory in \fod{} can contain multiple definitions. However, it is well-known that each such theory can be transformed into a theory that contains just a single definition~\cite{VANGELDER91}. This involves merging the different definitions and possibly renaming predicates to avoid the introduction of new loops. The necessity for this renaming step can be seen by comparing the following two theories:
$\theo$ consists of two separate definitions (one defining $p$ in terms of $q$ and the other defining $q$ in terms of $p$) and $\theo'$, which consists of a single definition that jointly defines both $p$ and $q$:
%
%
  \vspace{-1cm}
  \begin{multicols}{2}
    \begin{equation*}
    \label{eq:ex1}
    \theo = \left\{
      \begin{array}{c}
         \{ p \defarrow q. \} \\
         \{ q \defarrow p. \}
      \end{array}
      \right\}
    \end{equation*}  \break
    \begin{equation*}
    \label{eq:ex2}
    \theo' = \left\{
      \left\{
      \begin{array}{c}
         p \defarrow q. \\
         q \defarrow p.
      \end{array}
      \right\}
    \right\} 
    \end{equation*}
  \end{multicols}
  \noindent
The theory $\theo$ has two models, namely $\{\}$ and $\{p,q\}$, while $\theo'$ has $\{\}$ as its unique model. We therefore cannot simply merge the two definitions in $\theo$. The solution is to rename the predicates that are defined in (at least one of) these definitions, and then assert the equivalence between the old and the new predicates. 

More formally, for each definition $\Delta$, for each defined predicate $P$ in $\Delta$, we replace all occurrences of $P$ in $\Delta$ with a fresh unique predicate $P_\Delta$ and add the equivalence constraint $P \Leftrightarrow P_\Delta$.
Applying this merge procedure to $\theo$ yields the following theory:
\[
\theo'' = \left\{p \Leftrightarrow p', q \Leftrightarrow q', \left\{
      \begin{array}{c}
         p' \defarrow q. \\
         q' \defarrow p.
      \end{array}
      \right\}\right\}
\]
This avoids the introduction of additional loops and ensures $\theo''$ equivalent to the original $\theo$. 

We therefore from now on assume that the theory of the model expansion problem contains only a single definition $\Delta$. Each rule $r \in \Delta$ is of the form
\[ \forall x_1 [T_1]: \ldots: \forall x_n [T_n] : P(x_1,\ldots,x_n) \leftarrow \varphi.\]
We translate it to the following ASP rule $\alpha_4(r)$:
\[\asp{P}(\asp{x_1},\ldots,\asp{x_n})~\lif \alpha_3(\varphi), \asp{T_1}(\asp{x_1}), \ldots, \asp{T_n}(\asp{x_n}).\]
We then define $\alpha_4(\Delta)$ as $\{\alpha_4(r) \mid r \in \Delta\}$.

We now first show the correctness of this transformation in isolation, before combining it with previous results.
\begin{theorem}\label{th:def}
 Let $\Delta$ be a definition in vocabulary $V$ and let $S$ be a structure for $Open(\Delta)$, i.e, the set of all symbols in $V$ that do not appear in the head of any rule of $\Delta$. Then $\alpha_2(S) \cup \alpha_4(\Delta)$ has a unique answer set which coincides with the unique solution to $MX(V,S,\{\Delta\})$.
\end{theorem}
\begin{proof}
 A valid definition in \fod{} must be such that its well-founded model is always two-valued. Because the transformation from $\Delta$ to $\alpha_4(\Delta)$ introduces no additional loops over negation (in fact, it introduces no additional negations at all), the set of ASP rules $\alpha_4(\Delta)$ also has a two-valued well-founded model. It is well known that a two-valued well-founded model is also the unique stable model. Given this uniqueness result, the theorem now follows from the correctness of $\alpha_3$ (Theorem \ref{th:form}).
\end{proof}

Note that this theorem does not hold for structures $S$ that interpret some of the defined symbols of $\Delta$. Consider, for instance, the definition consisting only of the rule $p \defarrow{} \mathit{true}$ and the structure $S$ in which $p^S = false$. The problem $MX(V,S,\Delta)$ has no solutions, but $\alpha_2(S) = \{ \}$ and $\alpha_4(\Delta) = \{ p \lif \}$, which means that $\alpha_2(S) \cup \alpha_4(\Delta)$ has $\{p\}$ as an answer set.

For the same reason, we cannot simply combine $\alpha_4(\Delta)$ with the choice rules introduced by $\alpha_1$. To solve this problem, we will use the same renaming trick that we use to merge separate definitions.

\begin{definition}
Let $MX(V,S,\theo)$ be a model expansion problem in which the theory $\theo$ contains only a single definition $\Delta$, and let $\theo' = \theo \setminus \{\Delta\}$. Let $\Delta'$ be the result of replacing each defined predicate $P$ of $\Delta$ by a fresh predicate $P_\Delta$, and denote by $Eq$ the set of all equivalence constraints $\forall \vec{x}\colon P(\vec{x}) \Leftrightarrow P_\Delta(\vec{x})$ for defined predicates $P$. We define $\alpha(MX(V,S,\theo))$ as the following ASP program:
\[\alpha_1(MX(V,S,\cdot)) \cup \alpha_2(MX(V,S,\cdot)) \cup \alpha_3(MX(\cdot,S,\theo' \cup Eq)) \cup \alpha_4(MX(\cdot,S,\Delta')). \]
(For clarity, arguments have been replaced by $\cdot$ where they are irrelevant.)
\end{definition}

\begin{theorem}
For a model expansion problem $M = MX(V,S,\theo)$ in which the theory $\theo$ contains only a single definition $\Delta$, the solutions to $M$ coincide with the answer sets of $\alpha(M)$. 
\end{theorem}
\begin{proof}
Theorem \ref{th:form} already shows that all parts of the model expansion problem apart from the definition are correctly translated by $\alpha_1,\alpha_2$ and $\alpha_3$. Theorem \ref{th:def} shows that the definition $\Delta$ can be correctly translated by $\alpha_4(\Delta)$. The renaming of the defined predicates of $\Delta$ ensures that both can be combined without invalidating the correctness of either theorem.
\end{proof}

\subsection{Translating the graph coloring example}


We now show how $MX(V,S,\theo')$, with $\theo'$ the normalized theory from Example~\ref{ex:normal}, can be translated to ASP. 
This translation consists of four parts -- $\alpha_1$, $\alpha_2$, $\alpha_3$, $\alpha_4$ -- which correspond to the translation of the vocabulary $V$, the structure $S$, the constraints in $\theo'$, and the definitions in $\theo'$, respectively.


\begin{example}
\label{ex:translated}
The following is a translation of $M = MX(V,S,\theo')$ from Example~\ref{ex:normal} to ASP:
\begin{align*}
\alpha_1(M) \colon
&& \{\colorOf(C,X)\} \lif & \country(C), \ccolor(X). \\
&& \{\symBorder(C_1,C_2)\} \lif & \country(C_1), \country(C_2). \\
&& \delta_2(C) \lif & \aspcard\{ C,X \colon \colorOf(C,X),\ccolor(X) \}=1, \country(C).\\
&& \delta_1 \lif & \aspcard\{C\colon \delta_2(C), \country(C)\}=3. \\
&& \lif & \aspnot \delta_1. \\
\alpha_2(M) \colon
&&& \country(be).~\country(nl).~\country(lux).\\ 
&&& \border(nl,be).~\border(be,lux).
~\ccolor(\mathit{red}).~\ccolor(\mathit{blue}). \\
\alpha_3(M) \colon
&&  \delta_5(C_1, C_2, X) \lif & \aspnot \colorOf(C_1,X), \country(C_1), \country(C_2), \ccolor(X). \\
&&  \delta_5(C_1, C_2, X) \lif & \aspnot \colorOf(C_2,X), \country(C_1), \country(C_2), \ccolor(X). \\
&&  \delta_4(C_1,C_2) \lif & \aspcard\{X\colon \delta_5(C_1,C_2,X), \ccolor(X)\} = 2, \\
&&& \country(C_1), \country(C_2). \\
&& \delta_4(C_1,C_2) \lif & \aspnot \border(C_1,C_2), \country(C_1), \country(C_2). \\
&& \delta_3 \lif & \aspcard\{C_1, C_2\colon \delta_4(C_1,C_2), \country(C_1), \country(C_2) \}=9. \\
&& \lif & \aspnot \delta_3. \\
\alpha_4(M) \colon
&& \symBorder_\Delta(C_1,C_2) \lif & \border(C_1,C_2). \\
&& \symBorder_\Delta(C_1,C_2) \lif & \symBorder_\Delta(C_2,C_1). \\
&& \lif & \symBorder(C_1,C_2), \aspnot \symBorder_\Delta(C_1,C_2). \\
&& \lif & \aspnot \symBorder(C_1,C_2), \symBorder_\Delta(C_1,C_2).
\end{align*}
\end{example}

Example~\ref{ex:translated} and Example~\ref{ex:asp} are both ASP programs representing the same graph coloring problem.
$\alpha_1(M)$ in Example~\ref{ex:translated} corresponds to rules (\ref{ex:asp:generate}) and (\ref{ex:asp:card}) in Example~\ref{ex:asp}, $\alpha_2(M)$ is the same set of facts (\ref{ex:asp:facts}), $\alpha_3(M)$ corresponds to the constraint (\ref{ex:asp:constraint}), and $\alpha_4(M)$ corresponds to the rules (\ref{ex:asp:symborder}).

It is clear that the translation in Example~\ref{ex:translated} is a lot less succinct.
Firstly, the translation introduces a significant number of auxiliary predicates, both reification predicates $\delta_i$ and a renaming predicate $symBorder_\Delta$ for the defined predicate of the definition.
Secondly, the universal quantifications in the \fod{} specification lead to several cardinality aggregates not present in the original formulation.
Thirdly, the \fod{} implication that represents the graph coloring constraint is normalized into a nested disjunction (see Example~\ref{ex:normal}) and this leads to four translated rules, compared to the single rule (\ref{ex:asp:constraint}) in Example~\ref{ex:asp}.

\section{Implementation}

By implementing this translation, we created a new model expansion engine for \fod{}, called \folasp{}. It uses the syntax of the \idpthree{} system for its input and output, and uses \clingo{}~\cite{tplp/GebserKKS19} as back-end ASP solver.

In addition to the subset of \fod{} discussed in this paper, \folasp{} also supports minimum and maximum aggregates, arithmetic, function symbols in the head of a definition, partial interpretations, and partial functions. 
We thereby cover almost all language constructs supported by \idpthree{}, except for symbol overloading, chained (in)equalities and constructed types.
%
Besides the model expansion inference, \folasp{} also supports the optimization inference, which computes a model that minimizes the value of some integer objective function.

Where appropriate, \folasp{} uses the \fod{} type information to add ``guards'' of the form $type(X)$ for each variable $X$ to the bodies of the generated ASP rules. In other words, \folasp{} generates so-called \emph{safe} rules, which allows the resulting programs to be handled by ASP solvers such as \clingo{}.

\folasp{} is implemented in Python 3. Its source code is published on Gitlab.\footnote{\url{https://gitlab.com/EAVISE/folasp}}
%
We tested the correctness of the implementation by checking that the solutions produced by \folasp{} are accepted as such by \idpthree{}, and that, for optimization problems, the optimal objective values produced by \idpthree{} and \folasp{} were in agreement. 


\section{Experiments}
\label{sec:result}

In our experiments, we evaluate \folasp{} (commit \gitcommit{82ec7edc651f14e35790354acb63e4754b319195} on the development branch) using \clingo{} (version 5.4.0) as backend ASP solver.
This configuration is compared to two other approaches. The first comparison approach runs \idpthree{} (commit \gitcommit{4be3c797716a87d76e3eed3f08fb7921680b9972}) on the same \fodot{} specifications as taken as input by \folasp{}.
The second comparison runs \clingo{} (again version 5.4.0) on native ASP encodings of the same problems.

As benchmark set, we use the problem instances from the \emph{model-and-solve} track of the fourth ASP competition~\cite{conf/lpnmr/AlvianoCCDDIKKOPPRRSSSWX13}. Both IDP and \clingo{} participated in this competition, which means that we have---for these same problems---both \fod{} and ASP specifications already available, written by experts in both languages. As such, we believe these benchmarks provide a good opportunity for a fair comparison.

We used the scripts and \fodot{} specifications from the \idpthree{} submission to generate \fodot{} instances that are accepted by both \idpthree{} and \folasp{}.
For the ASP system \clingo{}, we used the native ASP encodings provided by the organizers of the fifth ASP competition~\cite{journals/ai/CalimeriGMR16} (which uses the same problem set) since \clingo{} could not parse the specifications from the fourth ASP competition.

\idpthree{} solves problems in NP. This covers all problems in the benchmark set, apart from the \instance{strategic companies} problem, which has a higher complexity. In the competition itself, the \idpthree{} team therefore solved this problem using a separate script to generate an exponentially sized search space, which was then given to \idpthree{}. Because this trick is not representative of how \idpthree{} is intended to be used in the real world, we decided to omit this benchmark from our experiments. Fourteen benchmark families remained with each (close to) thirty instances.
They cover a wide range of applications, from a simple reachability query over a transportation planning problem to optimizing the location of valves in an urban hydraulic network.

The experimental hardware consisted of a dual-socket Intel\textregistered{} Xeon\textregistered{} E5-2698 system with 512 GiB of RAM memory, with twenty hyper-threaded cores for each of the two processors.
To reduce resource competition, we run only twenty instances simultaneously, using twenty threads, for a total of ten per processor.
We employ a high memory limit of 64 GiB for each instance, as we observed that \clingo{} requires a significant amount of memory to solve instances generated by \folasp{}.\footnote{Note that even though twenty simultaneous instances utilizing 64 GiB of memory each is more than the total of 512 GiB of memory available, this worst-case scenario did not occur in practice and the machine did not run out of memory.}
We run the instances with a 6000 second timeout limit, but, to avoid imprecision at the timeout limit, we consider an instance unsolved if it takes more than 5000 seconds.
Optimization instances are considered solved when the last solution is proven to be optimal.

Runnable software, instance files, and detailed experimental results are made available at Zenodo.\footnote{\url{https://doi.org/10.5281/zenodo.4771774}}

\begin{figure}
  \includegraphics[width=\textwidth]{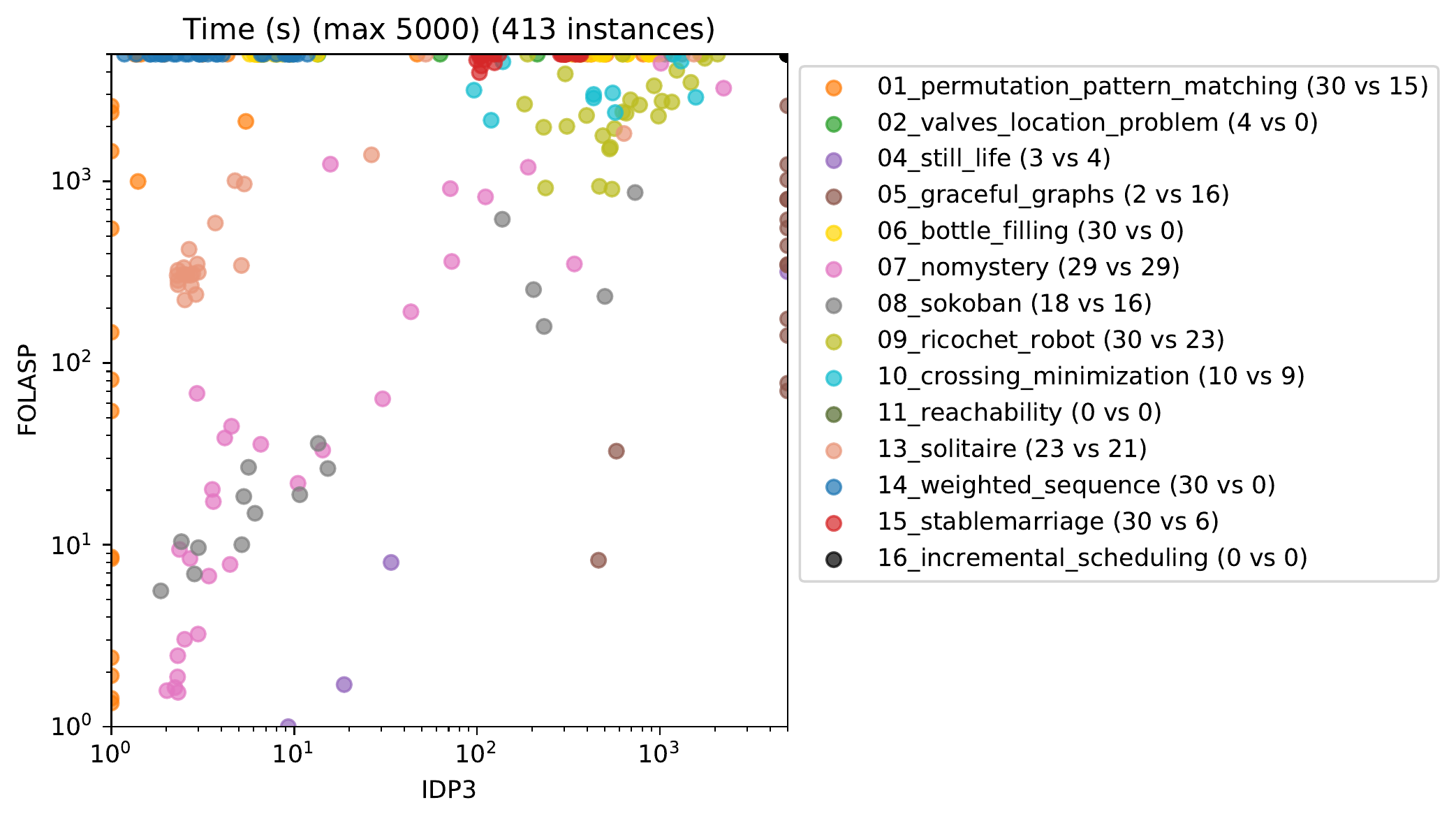}
  \caption{Scatter plot comparing \idpthree{} and \folasp{} runtime performance. ``($x$ vs $y$)'' denotes that \idpthree{} solved $x$ instances within the family, and \folasp{} $y$.}
  \label{fig:idpfolasptime}
\end{figure}

We compare the efficiency of the three approaches---\folasp{}, \idp{} and \clingo{}. 
Figure~\ref{fig:idpfolasptime} (best viewed in color) shows the time needed for both  \folasp{} and \idpthree{} to solve each instance. For benchmarks such as \instance{nomystery}, \instance{sokoban}, \instance{ricochet\_robot} and \instance{crossing\_minimization}, the performance of \folasp{} and \idp{} is about equal. For benchmarks such as \instance{permutation\_pattern\_matching}, \instance{valves\_location\_problem}, \instance{solitaire} and \instance{weighted\_sequence}, \idp{} clearly outperforms \folasp{}. This suggests that the specifications that were hand-crafted by the \idp{} team are indeed well-suited to this solver's particular characteristics, and less to those of \clingo{}. In addition, our translation of course introduces a number of artifacts, such as reification predicates, renaming predicates and cardinality aggregates, that may adversely impact performance as well. Interestingly, however, on the \instance{graceful\_graphs} benchmark, \folasp{} clearly outperforms \idpthree{}. 
Here, the inefficiencies of the translation are apparently overcome by the speed of \clasp{}. This highlights the usefulness of a translation such as ours: different benchmarks might be more suited for the architecture of different systems.



\begin{figure}
  \includegraphics[width=\textwidth]{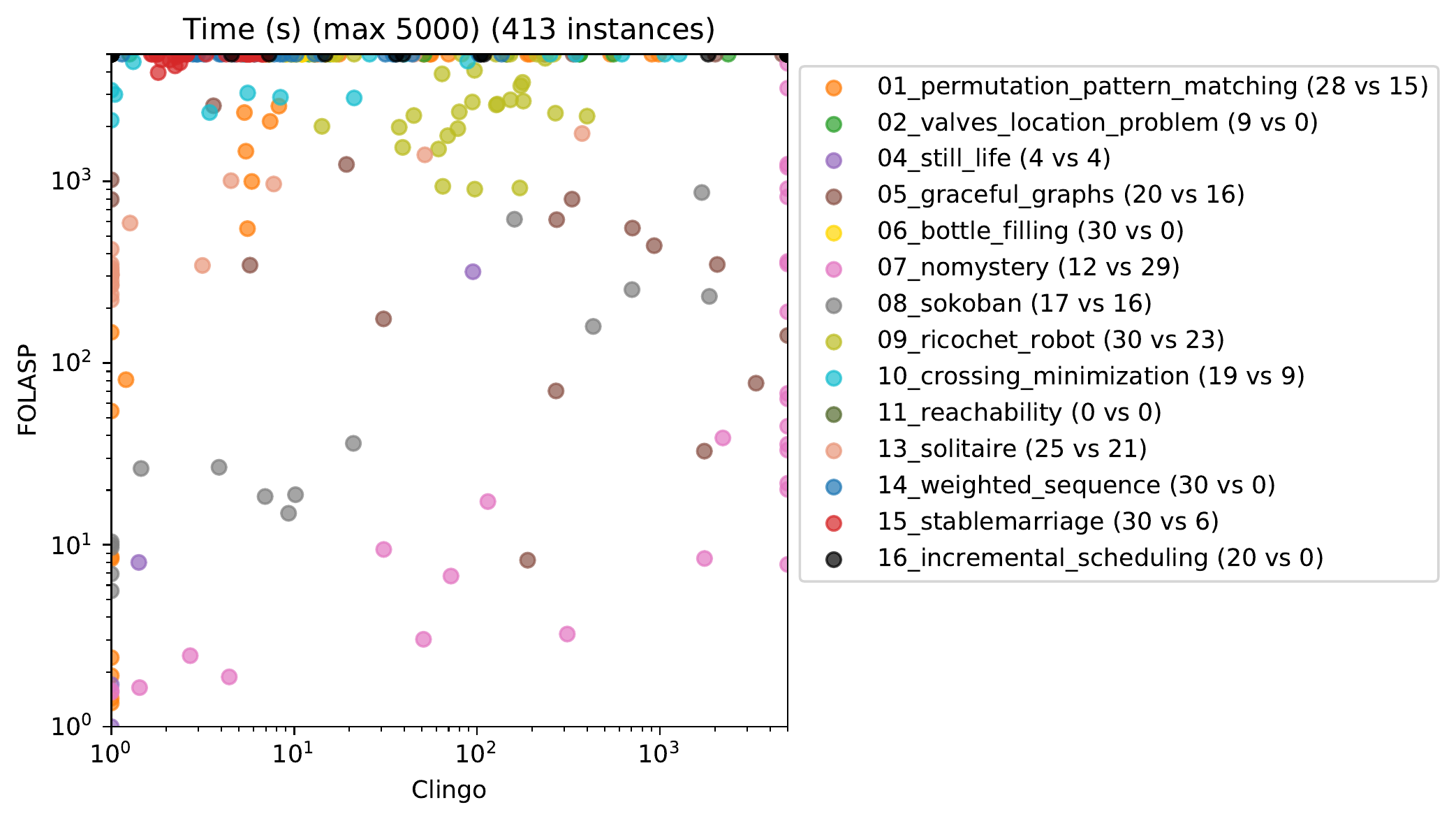}
  \caption{Scatter plot comparing \clingo{} and \folasp{} runtime performance. ``($x$ vs $y$)'' denotes that \clingo{} solved $x$ instances within the family, and \folasp{} $y$.}
  \label{fig:clingofolasptime}
\end{figure}

While the above experiments used different back-ends to handle precisely the same input, our next experiments (Figure~\ref{fig:clingofolasptime}) use the same \clingo{} back-end to solve the native ASP encodings as well as the translations that are automatically generated by \folasp{} from the \fod{} specification for the same benchmark. Here, we see that performance is about equal for \instance{sokoban}, \instance{still\_life} and \instance{graceful\_graphs}. For all other benchmarks apart from \instance{nomystery}, the native ASP specification significantly outperforms the automatic translation. This further confirms our earlier remark that the native version is able to better take advantage of the particular properties of \clasp{}, and that our translation's performance may suffer from the introduction of artificial predicates. 

In these experiments, the \instance{nomystery} benchmark is the odd one out, since the \folasp{} translation here significantly outperforms the native encodings.
One possible explanation for this is that the modeling style encouraged by \fod{} has computational properties different to those of typical ASP programs, and that the \fod{} style happens to be particularly well-suited to \instance{nomystery}. This would again point towards the value of a translation such as ours, but now from a different perspective: it is not only useful to be able to try out different back-ends with the same specification, but it is also useful to be able to run the same back-end with specifications that were written according to different paradigms.


\begin{figure}
  \includegraphics[width=\textwidth]{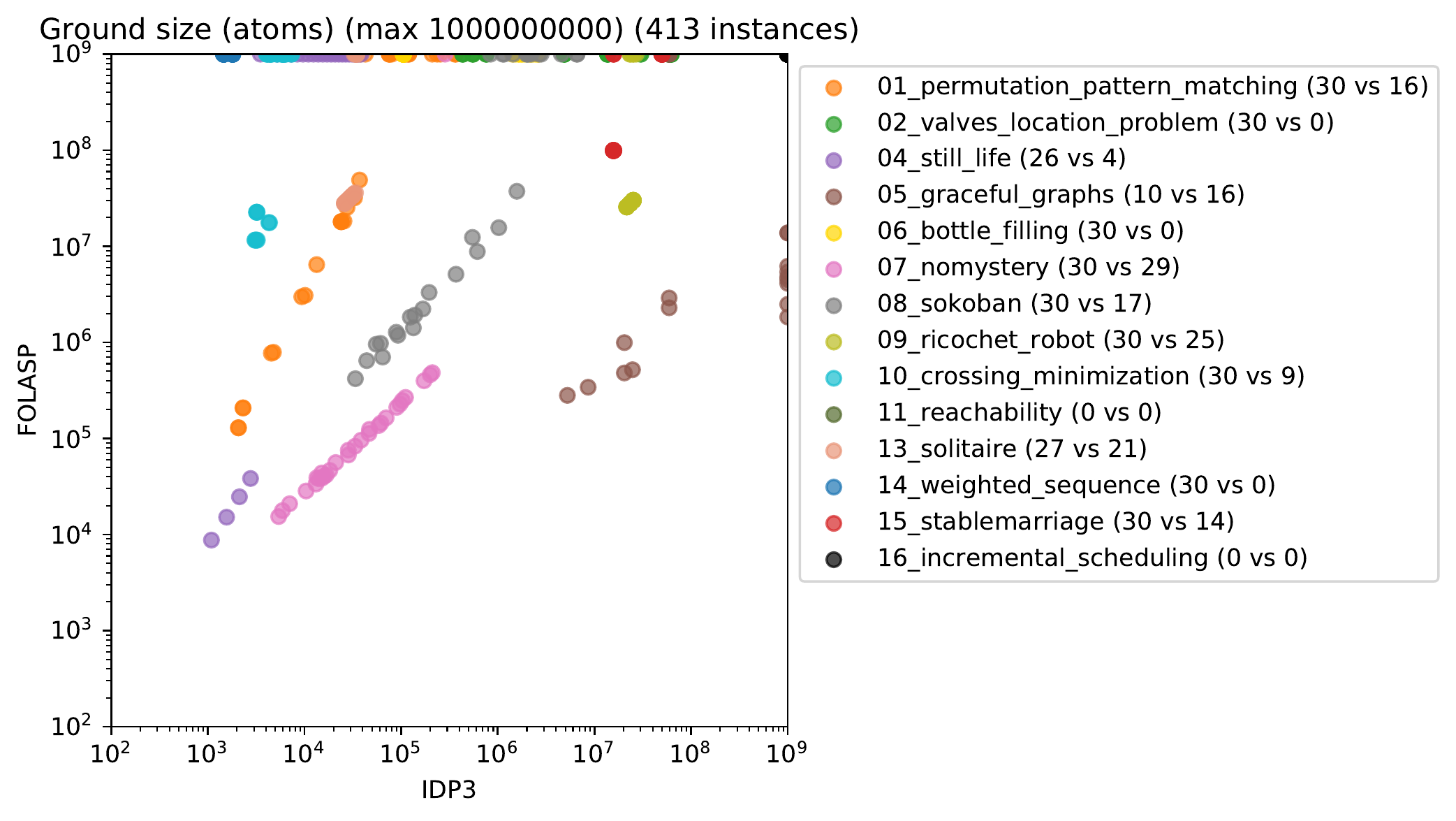}
  \caption{Scatter plot comparing \idpthree{} and \folasp{} ground size. ``($x$ vs $y$)'' denotes that \idpthree{} printed a ground size for $x$ instances, and \folasp{} for $y$.\label{fig:idpfolaspgrounding}}
\end{figure}

\begin{figure}
  \includegraphics[width=\textwidth]{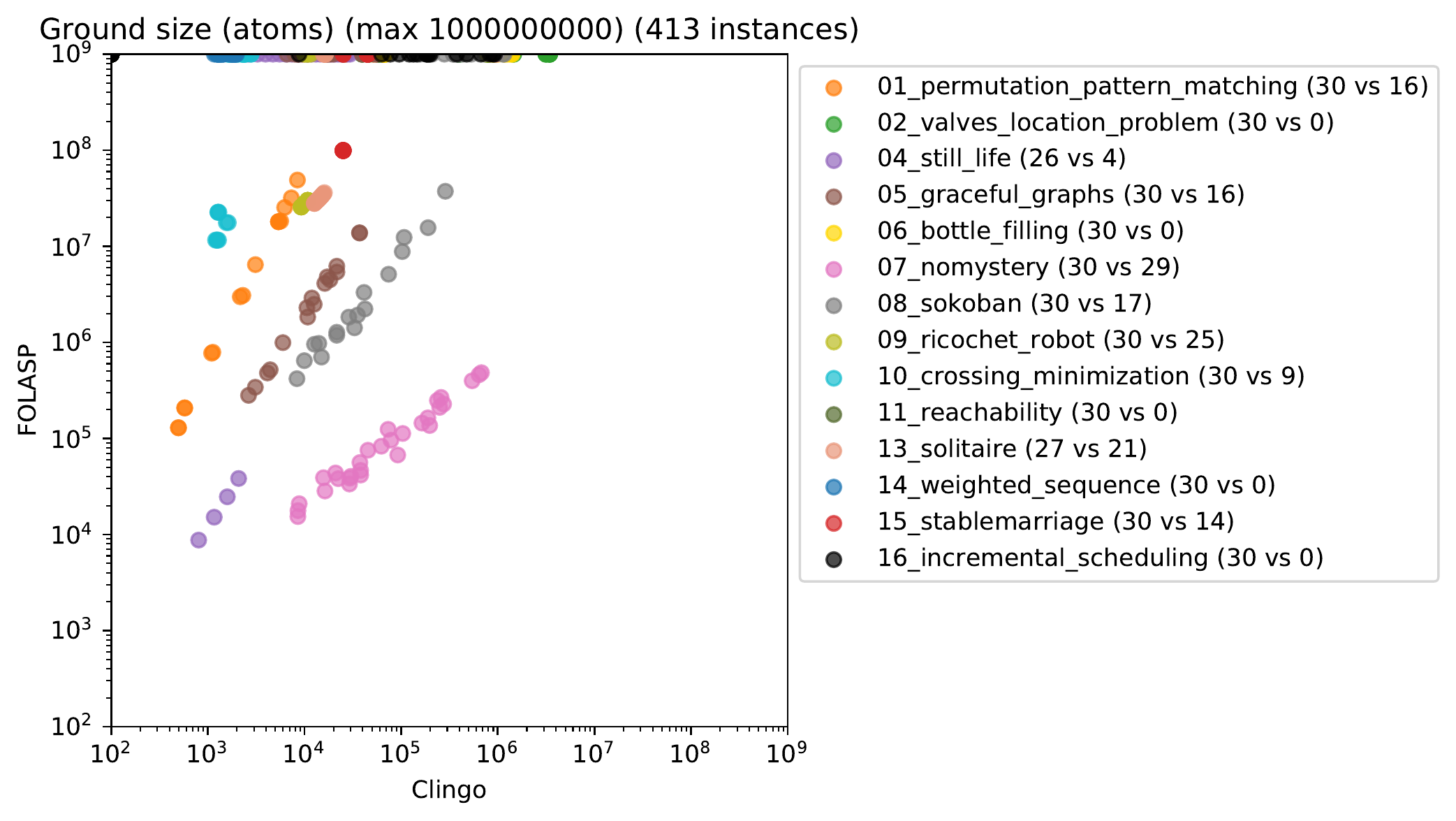}
  \caption{Scatter plot comparing \clingo{} and \folasp{} ground size. ``($x$ vs $y$)'' denotes that \clingo{} printed a ground size for $x$ instances, and \folasp{} for $y$.
  \label{fig:clingofolaspgrounding}}
\end{figure}

In our discussion of the previous experiments, we have hypothesised that the artifacts of our translation may make \folasp{}'s specifications harder to solve. To further investigate this, Figures~\ref{fig:idpfolaspgrounding} and~\ref{fig:clingofolaspgrounding} compare the size of the ground programs for the three approaches.
This ground size is measured as the number of atoms in the ground program, for both the ASP-solving approaches (\folasp{} and \clingo{}) and \idpthree{}, with the caveat that a ground ASP program and a ground \fod{} specification may still be quite different. For instance, \idpthree{} retains non-Boolean \emph{CP variables}~\cite{Cat13modelexpansion} in its grounding, though we did switch off \emph{lazy grounding}~\cite{jair/CatDBS15}.

For most benchmarks, the ground size for \folasp{} is indeed significantly larger than the ground size for both \idpthree{} and \clingo{}. 
Moreover, the ground size seems to correlate roughly with performance.
For instance, \folasp{} outperforms \idpthree{} and \clasp{} on \instance{graceful\_graphs} and \instance{nomystery}, respectively, and also has the  smaller ground sizes on these benchmarks.
On other benchmarks, such as \instance{bottle\_filling} or \instance{weighted\_sequence}, \folasp{} actually hit the 64 GiB memory limit during grounding.

These observations appear to confirm our hypothesis that artifacts introduced by the translation, such as auxiliary predicates, cardinality aggregates and extra rules, are a main source of poor performance. Future work may focus on how to tweak the translation such that the ground size can be reduced.



\section{Conclusion}
\label{sec:conclude}
To solve real-world problems using declarative methods, both a suitable modeling language and a suitable solver are needed. The Answer Set Programming community has converged on the \aspcore{} standard as a common modeling language. However, while such a common language is a great driver for technological progress, it is not necessarily well-suited for all applications.

The \fod{} language may provide an interesting alternative. It builds on classical first-order logic, which may make it easier to use for domain experts who are already familiar with FO, and which may make it easier to integrate with other FO-based languages. However, it is only supported by a few solvers, which restricts the applications for which \fod{} can be used in practice.

In this paper, we aim to provide more flexibility:  by presenting a translation of \fod{} model expansion problems to \aspcore{}, we both extend the range of solvers for \fod{} and enable the use of \fod{} as an alternative modeling language for these solvers. In this way, we stimulate technological progress in solver development and in the development of applications.
We implemented our approach in the \folasp{} tool, which, to the best of our knowledge, is the first tool to offer a full translation from \fod{} to ASP for both model expansion and optimization.

In our experimental evaluation, we used benchmarks from the ASP competition to verify that the results computed by \folasp{} are indeed correct. We also compared the performance of running \clingo{} on the \folasp{} translation of an \fod{} specification to two alternatives:
\begin{itemize}
    \item directly running the \idpthree{} system on the \fod{} specification;
    \item running \clasp{} directly on a native ASP specification.
\end{itemize}
In general, our experiments confirmed what one would typically expect, namely that the best performance is obtained by running a specification that was native to a particular solver on that solver. However, the experiments also showed that, for a number of benchmarks, our translation-based approach is actually able to match or even, in rare cases, outperform the native approaches. This demonstrates the usefulness of our translation also from a computational perspective: a specification that performs poorly with one solver, may be more efficient when translated to the input language of another solver. 

Our experiments also demonstrate that, in cases where the translation performs significantly worse than the native solutions, the grounding size often appears to play an important role. Future work will therefore focus on further optimising the translation to reduce the overhead it introduces.

In summary, the main contribution of our work is to provide increased flexibility, both in choice of specification language and in choice of solver. We believe that this will be useful to drive technological progress, to develop real-world applications using the best tools for the job, and to allow cross-fertilisation between different research groups.




\bibliographystyle{acmtrans}
\bibliography{krrlib,folasp}

\end{document}